\documentclass[10pt, conference, compsocconf]{IEEEtran}
\usepackage{epsfig}
\usepackage{graphicx}
\usepackage{amsmath}
\usepackage{amssymb}
\usepackage{enumerate}
\usepackage{algorithm}
\usepackage{algorithmic}
\usepackage{amsthm}
\usepackage{subfigure}
\usepackage{epstopdf}
\theoremstyle{plain}
\newtheorem*{theorem*}{Theorem}
\newtheorem{theorem}{Theorem}[]

\newtheorem{lemma}{Lemma}

\DeclareMathOperator*{\argmin}{arg\,min}
\usepackage{thmtools, thm-restate}
\ifCLASSINFOpdf
\else
\fi
\begin{document}
%
\title{Robust PCA via Nonconvex Rank Approximation}

\author{\IEEEauthorblockN{Zhao Kang, Chong Peng, Qiang Cheng}
\IEEEauthorblockA{ Department of Computer Science, Southern Illinois University, Carbondale, IL 62901, USA\\\{zhao.kang, pchong, qcheng\}@siu.edu}}


%


\maketitle

\begin{abstract}
Numerous applications in data mining and machine learning require recovering a matrix of minimal rank. Robust principal component analysis (RPCA) is a general framework for handling this kind of problems. Nuclear norm based convex surrogate of the rank function in RPCA is widely investigated. Under certain assumptions, it can recover the underlying true low rank matrix with high probability. However, those assumptions may not hold in real-world applications. Since the nuclear norm approximates the rank by adding all singular values together, which is essentially a $\ell_1$-norm of the singular values, the resulting approximation error is not trivial and thus the resulting matrix estimator can be significantly biased. To seek a closer approximation and to alleviate the above-mentioned limitations of the nuclear norm, we propose a nonconvex rank approximation. This approximation to the matrix rank is tighter than the nuclear norm. To solve the associated nonconvex minimization problem, we develop an efficient augmented Lagrange
multiplier based optimization algorithm. Experimental results demonstrate that our method outperforms current state-of-the-art algorithms in both accuracy and efficiency. 
\end{abstract}


%
\IEEEpeerreviewmaketitle

\section{Introduction}
In many machine learning and data mining applications, the dimensionality of data is very high, such as digital images, video sequences, text documents, genomic data, social networks, and financial time series. Data mining on such data sets is challenging due to the curse of dimensionality. Dimensionality reduction techniques, which project the original high-dimensional feature space to a low-dimensional space, have been extensively explored. Among them, principal component analysis (PCA) \cite{jolliffe2002principal}, which finds a small number of orthogonal basis vectors that characterize most of the variability of the data set, is well established and commonly used. However, PCA may fail spectacularly even when a single grossly corrupted entry exists in the data. To enhance its robustness to outliers or corrupted observations, early attempts on robust PCA (RPCA) have been made \cite{xu1995robust}, \cite{croux2000principal}, \cite{de2001robust}, \cite{de2003framework}, \cite{croux2005high}. Nevertheless, none of these algorithms yields a solution in polynomial-time with strong performance guarantees under broad conditions. 

Due to the seminal work of \cite{wright2009robust}, \cite{candes2011robust}, a more recent version of RPCA becomes popular these days. The idea is to recover a low-rank matrix $L$ from highly corrupted observations $X=L+S\in \mathcal{R}^{m\times n}$. Entries in the sparse component $S$ can have arbitrarily large magnitude. This has numerous applications ranging from recommender system design to anomaly detection in dynamic networks.  For example, for videos and face images under varying illumination, the background and underlying clean face image are regarded as the low-rank component while the moving objects and shadows represent the sparse part \cite{candes2011robust}; common words in a collection of text documents can be captured by a low-rank matrix while the few words that distinguish each document from others can be represented by a sparse matrix \cite{min2010decomposing}. 

Mathematically, this kind of problem can be modeled as 
\begin{equation}
\label{originalRPCA}
\min_{L, S} rank(L)+\lambda\|S\|_0\quad s.t.\quad X=L+S,
\end{equation}
where 
 $\lambda$ a weight parameter. Unfortunately, (\ref{originalRPCA}) is generally an NP-hard problem. By relaxing the nonconvex rank function and the $\ell_0$-norm into the nuclear norm and $\ell_1$-norm respectively, a convex formulation can be yielded
\begin{equation}
\label{cnxRPCA}
\min_{L, S} \|L\|_*+\lambda\|S\|_1\quad s.t.\quad X=L+S,
\end{equation}
 where $\|L\|_*=\sum_i \sigma_i(L)$; i.e., the nuclear norm of $L$ is the sum of its singular values, and $\|S\|_1=\sum_{ij}|S_{ij}|$. Under incoherence assumptions, both low-rank and sparse components can be recovered exactly with an overwhelming probability  \cite{candes2011robust}. 

Despite its convex formulation and ease of optimization, RPCA in (\ref{cnxRPCA}) has two major limitations. First, the underlying matrix may have no incoherence guarantee \cite{candes2011robust} in practical scenarios, and the data may be grossly corrupted. Under these circumstances, the resulting global optimal solution to (\ref{cnxRPCA}) may deviate significantly from the truth. Second, RPCA shrinks all the singular values equally. The nuclear norm is essentially an $\ell_1$ norm of the singular values and it is well known that $\ell_1$ norm has a shrinkage effect and leads to a biased estimator \cite{fan2001variable}, \cite{zhang2010nearly}. This implies that the nuclear norm over-penalizes large singular values, and consequently it may only find a much biased solution.  Nonconvex penalties to $\ell_1$ norm such as smoothly clipped absolute deviation penalty \cite{fan2001variable}, minimax concave penalty \cite{zhang2010nearly}, capped-$\ell_1$ regularization \cite{gong2012multi}, and truncated $\ell_1$ function \cite{xiang2013efficient} have shown that they provide better estimation accuracy and variable selection consistency \cite{wang2014optimal}. Recently, nonconvex relaxations to the nuclear norm have received increasing attention \cite{kang2015cikm}. Variations of the nuclear norm, e.g., weighted
nuclear norm \cite{gu2014weighted}, \cite{zhong2015nonconvex}, singular value thresholding
\cite{cai2010singular}, and truncated nuclear norm \cite{hu2013fast} are proposed and outperform the standard nuclear norm. However, their applications are still quite limited and they are often designed for specific applications.    



In this paper, we propose a novel nonconvex function to directly approximate the rank, which provides a tighter approximation than the nuclear norm does. This is crucial to reveal the rank in low-rank matrix estimation. To solve this nonconvex model, we devise an Augmented Lagrange
Multiplier (ALM) based optimization algorithm. Theoretical convergence
analysis shows that our iterative optimization at least 
converges to a stationary point. Extensive experiments on three representative applications confirm the advantages of our approach.

\section{Related Work}
The convex approach to RPCA in (\ref{cnxRPCA}) has been studied thoroughly. It is proved that when the locations of nonzero entries of $S$ are uniformly distributed, and when the rank of $L$ and the sparsity of $S$ satisfy some mild conditions, $L$ and $S$ can be exactly recovered with a high probability \cite{candes2011robust}. In the literature, numerous algorithms have been developed to solve (\ref{cnxRPCA}), e.g., SVT \cite{cai2010singular}, APGL \cite{toh2010accelerated}, FISTA \cite{beck2009fast}, and ALM \cite{lin2010augmented}. Among them, ALM based approach is the most popular. Although the theory is elegant, convex technique is still computationally quite expensive and has poor convergence rate \cite{netrapalli2014non}. Furthermore, (\ref{cnxRPCA}) breaks down when large errors concentrate only on a number of columns of $S$ \cite{xu2012robust}, \cite{liu2013robust}.   

To incorporate the spatial connection information of the sparse elements, $\ell_{2,1}$-norm is introduced in outlier pursuit \cite{xu2010robust}, \cite{mccoy2011two}:
\begin{equation}
\label{l21RPCA}
\min_{L, S} \|L\|_*+\lambda\|S\|_{2,1}\quad s.t.\quad X=L+S.
\end{equation}
Here, $\|S\|_{2,1}:=\sum_{j=1}^n \sqrt{\sum_{i=1}^m S_{ij}^2}$ can detect outliers with column-wise sparsity, while $\|S\|_1$ treats each entry independently. Theoretical analysis on this model is difficult. In this model, only the column space of $L$ and the column support of $S$ can be exactly recovered \cite{chen2011robust}, \cite{xu2012robust}. When the rank
of the intrinsic matrix $r$ is comparable to the number of samples $n$, the working range of outlier pursuit is limited.

To alleviate the deficiency of convex relaxations, capped norm based nonconvex RPCA (CNorm) has been proposed and it solves the following problem \cite{sun2013robust}:
\begin{equation}
\begin{split}
\label{cnorm}
&\min_{L,S}\quad \frac{1}{\theta_1}\|L\|_*+\frac{1}{\theta_2}\|S\|_1-\left[\frac{1}{\theta_1}P_1(L)+\frac{1}{\theta_2}P_2(S)\right]\\
&s.t. \quad \|X-L-S\|_F^2\leq\delta^2,
\end{split}
\end{equation}
where $P_1(L)=\sum_{i=1}^{\min(m,n)} \max(\sigma_i(L)-\theta_1, 0)$, $P_2(S)=\sum_{ij} \max(|S_{ij}|-\theta_2, 0)$ for some small parameters $\theta_1$, $\theta_2>0$, and $\delta$ denotes the level of Gaussian noise. If all singular values of $L$ are greater than $\theta_1$ and all absolute values of $S$ elements are greater than $\theta_2$, then the objective function in (\ref{cnorm}) falls back to (\ref{originalRPCA}). However, it is hard to provide any convergence guarantee about this nonconvex method. More importantly, as we will show in the experimental part, it cannot deal with large scale data well. 

By combining the simplicity of PCA and elegant theory of convex RPCA, a recent paper has proposed a new nonconvex RPCA \cite{netrapalli2014non}. The idea is to project the residuals onto the set of low-rank and sparse matrices alternatively. Specifically, it proceeds in $r$ (the desired rank of $L^*$) stages, and compute rank-$k$ projection in each stage, where $k\in\{1, 2, ..., r\}$. During this process, sparse errors are suppressed by discarding matrix elements with large approximation errors. This method enjoys several nice properties, including low complexity, global convergence guarantee, fast convergence rate, and theoretical guarantee for exact recovery of the low-rank matrix. However, it needs the knowledge of three parameters: sparsity of $S^*$, incoherence of $L^*$, and rank $r$ of $L^*$. Such knowledge is not always readily available.

\section{Proposed algorithm}
In this section, we present a novel matrix rank approximation, and propose a nonconvex RPCA algorithm.
\subsection{Problem formulation}
Consider the general framework for RPCA 
\begin{equation}
\min_{L, S} \|L\|_\gamma+\lambda \|S\|_l \quad s.t. \quad X=L+S,
\label{originalproblem}
\end{equation} 
where $\|\cdot\|_\gamma$ denotes a rank approximation which we term $\gamma$-norm, and $\|\cdot\|_l$ represents a proper norm of noise and outliers. 
\begin{figure}[ht]
\centering

\includegraphics[width=.5\textwidth]{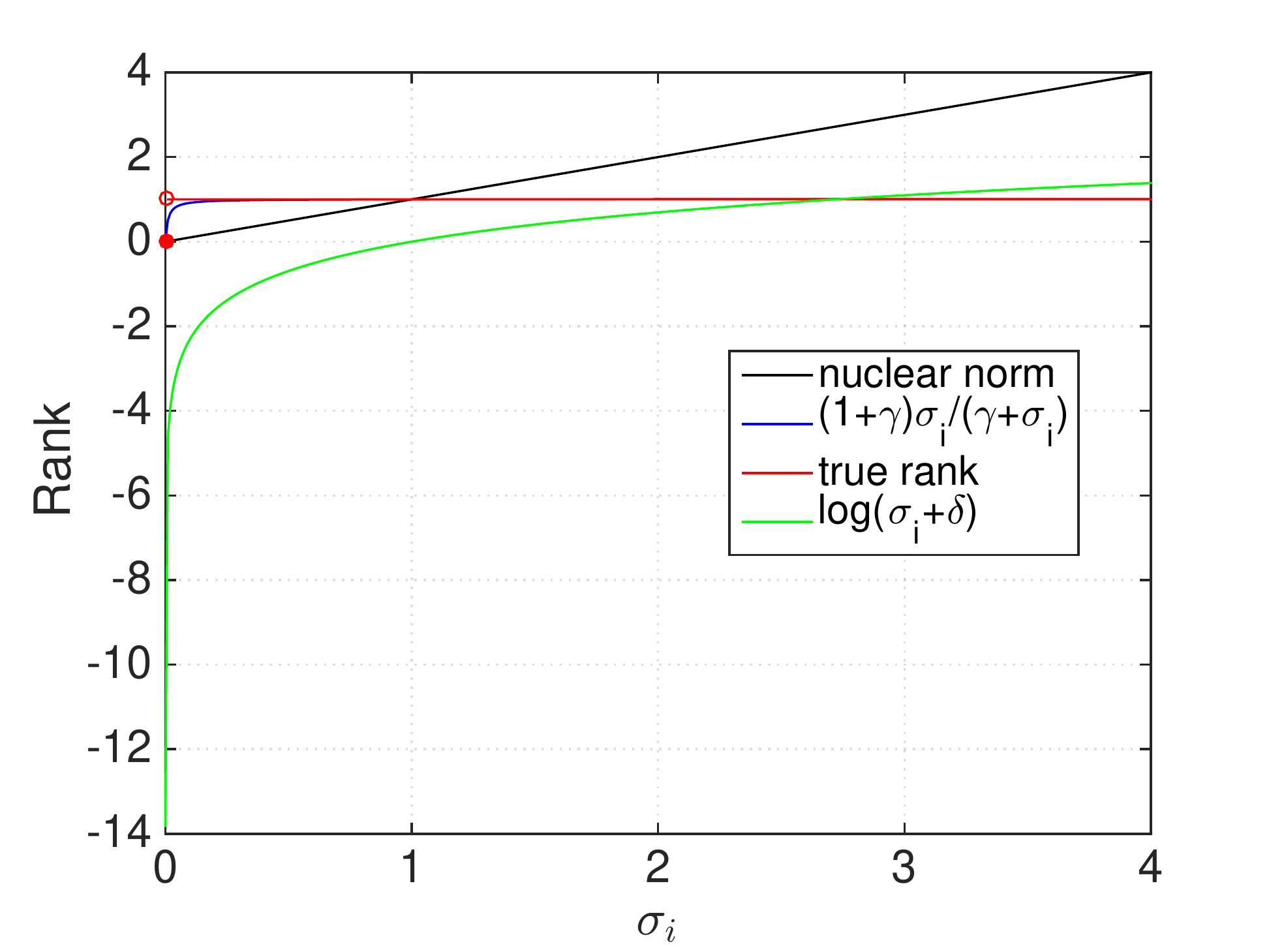}
\caption{The contribution of different functions to the rank with respect to a varying singular value. The true rank is 1 for nonzero $\sigma_i$. }
\label{rankcomp}
\end{figure}

We define $\gamma$-norm of matrix $L$ as 
\begin{equation}
\|L\|_\gamma=\sum_{i}\frac{(1+\gamma)\sigma_i(L)}{\gamma+\sigma_i(L)}, \quad \gamma>0.
\end{equation}
It can be observed that $\lim\limits_{\gamma \rightarrow 0}\|L\|_\gamma=rank(L)$, $\lim\limits_{\gamma \rightarrow \infty}\|L\|_\gamma=\|L\|_*$, and it coincides with true rank with $\sigma_i(L), i=1, \cdots, \min(m, n),$ being all 0 and all 1. Furthermore, $\|L\|_\gamma$ is unitarily invariant, that is, $\|L\|_{\gamma} = \| U L V \|_{\gamma}$ for any orthonormal $U \in {\mathcal{R}}^{m \times m}$ and $V \in {\mathcal{R}}^{n \times n}$. Certainly it is not a real norm. Figure \ref{rankcomp} plots several rank relaxations in the literature. Among them, a log-determinant function, $\mathrm{logdet}(L+\epsilon I)$, where $\epsilon$ is a very small constant (e.g., $10^{-6}$), has been well studied \cite{fazel2003log}. As we can see, our formulation ($\gamma=0.01$ is used in this figure and our experiments) closely matches the true rank, while the nuclear norm deviates considerably when the singular values depart from 1.  As a result, the proposed $\gamma$-norm overcomes the imbalanced penalization by different singular values in convex nuclear norm. On the other hand, (\ref{originalproblem}) is a nonconvex formulation, which is usually difficult to optimize. In the next section, we design an effective algorithm to solve it. 

\subsection{Optimization}
For problem (\ref{originalproblem}), by introducing a Lagrange multiplier $Y$ and a quadratic penalty term, we can remove the equality constraint and construct the augmented Lagrangian function:
\begin{equation}
\label{lagrang}
\begin{split}
&\mathcal{L}(L, S, Y,\mu) =\|L\|_\gamma+\lambda \|S\|_l+  \\
&\langle Y,L+S-X\rangle +\frac{\mu}{2}\|L+S-X\|_F^2,
\end{split}
\end{equation}
where $\langle \cdot,\cdot \rangle$ is the inner product of two matrices, that is, $\langle A, B\rangle=tr(A^TB)$, and $\mu$ is a positive parameter. 
An iterative approach is applied to update $L$, $S$ and $Y$ iteratively. At the $(t+1)$th step, we update $L^{t+1}$ by solving the following subproblem:
\begin{equation}
L^{t+1}=\argmin_L \hspace{.1cm} \|L\|_{\gamma}+\frac{\mu^t}{2}\left\|L-( X-S^t-\frac{Y^t}{\mu^t})\right\|_F^2.
\label{upZ}
\end{equation}
To solve (\ref{upZ}), we first develop the following theorem and provide the proof in Appendix A.
\begin{theorem}
\label{firsthm}
Let $A=U\Sigma_A V^T$ be the SVD of $A\in \mathbf{\mathcal{R}}^{m\times n}$ and $\Sigma_A=diag(\sigma_A)$. Let $F(Z)=f\circ\sigma(Z)$ be a unitarily invariant function 
 and $\mu>0$ 
. Then an optimal solution to the following problem
\begin{equation}
\min_Z F(Z)+\frac{\mu}{2}\left\|Z-A\right\|_F^2,
\label{theoremprob}
\end{equation}
is $Z^*= U\Sigma_Z^*V^T$, where $\Sigma_Z^*=diag(\sigma^*)$ and  $\sigma^* = \mathrm{prox}_{f, \mu} (\sigma_{A})$. Here $\mathrm{prox}_{f, \mu} (\sigma_{A})$ is the Moreau-Yosida operator, defined as 
\begin{equation}
\label{scalar}
\mathrm{prox}_{f, \mu} (\sigma_A) := \argmin_{\sigma\geq0} f(\sigma) + \frac{\mu}{2}\|\sigma - \sigma_A\|_2^2.
\end{equation}
\end{theorem}

In our case, the new objective function in (\ref{scalar}) is a combination of concave and convex functions. This intrinsic structure motivates us to use difference of convex (DC) programing \cite{tao1997convex}. DC algorithm decomposes a nonconvex function as the difference of two convex functions and iteratively optimizes it by linearizing the concave term at each iteration. At the $(k+1)$th inner iteration,   
\begin{equation}
\label{vectorized}
\sigma^{k+1}=\argmin_{\sigma\geq 0} \quad \langle w_k,\sigma \rangle+ \frac{\mu^t}{2}\|\sigma-\sigma_A\|_2^2,
\end{equation}
which admits a closed-form solution 
\begin{equation}
\label{optisigma} 
\sigma^{k+1}=(\sigma_A-\frac{\omega_k}{\mu^t})_+,
\end{equation}
where $\omega_k=\partial f(\sigma^k)$ is the gradient of $f(\cdot)$ at $\sigma^k$ and $U diag\lbrace\sigma_A\rbrace V^T$ is the SVD of $X-S^{t}-\frac{Y^t}{\mu^t}$. After a number of iterations, it converges to a local optimal point $\sigma^*$. Then $L^{t+1}=U diag\lbrace\sigma^*\rbrace V^T$.

For $S$ optimization, 
\begin{equation}
S^{t+1}= \argmin_S \hspace{.05cm}\lambda \left\|S\right\|_l+\frac{\mu^t}{2}\left\|S-(X-L^{t+1}-\frac{Y^t}{\mu^t})\right\|_F^2.
\label{solveS}
\end{equation}
Depending on the choice of $l$, we obtain different closed-form solutions to the above subproblem. According to the result in \cite{yuan2006model} which is also given as Lemma \ref{21norm} in Appendix B, for $\|S\|_{2,1}$ norm,
\begin{eqnarray}
\label{error21}
[S^{t+1}]_{:,i}=\left\{
\begin{array}{ll} \frac{\left\|Q_{:,i}\right\|_2-\frac{\lambda}{\mu^t}}{\left\|Q_{:,i}\right\|_2}Q_{:,i}, & \mbox{if $\left\|Q_{:,i}\right\|_2>\frac{\lambda}{\mu^t}$};\\
0, & \mbox{otherwise,}
\end{array}\right.
\end{eqnarray}
where $Q=X-L^{t+1}-\frac{Y^t}{\mu^t}$ and $\left[S^{t+1}\right]_{:,i}$ is the $i$-th column of $S^{t+1}$. 

When modeled by $\|S\|_1$ norm, based on Lemma \ref{1norm}, we have 
\begin{equation}
\label{error1}
\left[S^{t+1}\right]_{ij}=
\max\left(|Q_{ij}|-\frac{\lambda}{\mu^t},0\right) sign(Q_{ij}).
\end{equation}

The updates of $Y$ and $\mu$ are standard:
\begin{equation}
\label{multi}
Y^{t+1}=Y^t+\mu^t (L^{t+1}-X+S^{t+1}),
\end{equation}
\begin{equation}
\label{upmu}
\mu^{t+1}=\rho\mu^t,
\end{equation}
where $\rho>1$. 
The complete procedure is outlined in Algorithm 1. 
\begin{algorithm}[tb]
\small
   \caption{Solving problem (\ref{originalproblem}) }
   \label{alg:ncxpca}
    {\bfseries Input:} data matrix $X\in \mathbf{\mathcal{R}}^{m\times n}$, parameters $\lambda>0, \mu^0>0$, and $\rho>1$.\\
 {\bfseries Initialize:} $S=0$, $ Y=0$.\\
  {\bfseries REPEAT}
\begin{algorithmic}[1]
  \STATE Update $L$ by (\ref{upZ}). 
   \STATE Solve $S$ by either (\ref{error21}) or (\ref{error1}) according to  $l$.
\STATE Update $Y$ and $\mu$ by (\ref{multi}) and (\ref{upmu}), respectively. 

\end{algorithmic}
   \textbf{UNTIL} {converge.}
\end{algorithm}

\section{Convergence analysis}
Convergence analysis of nonconvex optimization problem is usually difficult. In this section, we will show that our algorithm has at least a convergent subsequence which tends to a stationary point. While the final solution might not be a globally optimal one, all our
experiments show that our algorithm converges to a solution
that produces promising results. 

For convenience, we write $\|L\|_\gamma$ as $F(L)$ in (\ref{lagrang})  
\begin{equation}
\label{newlagrang}
\begin{split}
&\mathcal{L}(L, S, Y,\mu) =F(L)+\lambda \|S\|_l+  \\
&\langle Y,L+S-X\rangle +\frac{\mu}{2}\|L+S-X\|_F^2,
\end{split}
\end{equation}

\begin{lemma}
\label{lemma11}
The sequence $\{Y^t\}$ is bounded.
\end{lemma}
\begin{proof}
$S^{t+1}$ satisfies the first-order necessary local optimality condition,
\begin{equation}
\begin{split}
\label{bound1}
&0\in\partial_S \mathcal{L}\left(L^{t+1}, S, Y^t, \mu^{t}\right)|_{S^{t+1}}\\
=&\partial_S \left(\lambda \|S\|_l\right)|_{S^{t+1}}+Y^t+\mu^{t}\left(L^{t+1}-X+S^{t+1}\right)\\
=&\partial_S \left(\lambda\|S\|_l\right)|_{S^{t+1}}+Y^{t+1}.
\end{split}
\end{equation}
For $\|S\|_1$ case, since  $||S||_1$ is nonsmooth at $S_{ij}=0$, we redefine subgradient $\left[\partial_S\|S\|_1\right]_{ij}=0$ if $S_{ij}=0$. Then $0\leq \| \partial_S \|S\|_1 \|_F^2\leq mn $, hence $\partial_S (\lambda\|S\|_1)|_{S^{t+1}}$ is bounded. Similarly, it can be shown that $\partial_S (\lambda\|S\|_{2,1})|_{S^{t+1}}$ is also bounded. Thus $\{Y^t\}$ is bounded.
\end{proof}

\begin{lemma}
\label{lemma12}
$\{L^t\}$ and $\{S^t\}$ are bounded if $\sum\limits_{t=1}^{\infty} \frac{\mu^{t}+\mu^{t+1}}{(\mu^t)^2}<\infty$.
\end{lemma}
\begin{proof}
With some algebra, we have the following equality
\begin{equation}
\begin{split}
& \mathcal{L}\left(L^{t},S^{t},Y^{t},\mu^{t}\right)\\
=&\mathcal{L}\left(L^{t},S^{t},Y^{t-1}, \mu^{t-1}\right)+\frac{\mu^t-\mu^{t-1}}{2}\|L^t-X+S^t\|_F^2\\
&+Tr[(Y^t-Y^{t-1})(L^t-X+S^t)]\\
=& \mathcal{L}\left(L^{t},S^{t},Y^{t-1}, \mu^{t-1}\right)+\frac{\mu^t+\mu^{t-1}}{2(\mu^{t-1})^2}\|Y^t-Y^{t-1}\|_F^2.
\end{split}
\end{equation}
Then, 
\begin{equation}
\begin{split}
& \mathcal{L}\left(L^{t+1},S^{t+1},Y^{t}, \mu^{t}\right)\\
\leq& \mathcal{L}(L^{t+1},S^{t}, Y^{t}, \mu^t)\\
\leq&\mathcal{L}(L^{t}, S^{t}, Y^{t}, \mu^t)\\
\leq & \mathcal{L}\left(L^{t},S^{t},Y^{t-1}, \mu^{t-1}\right)+\frac{\mu^t+\mu^{t-1}}{2(\mu^{t-1})^2}\|Y^t-Y^{t-1}\|_F^2. \label{lasteq}
\end{split}
\end{equation}
Iterating the inequality chain (\ref{lasteq}) $t$ times, we obtain
\begin{equation}
\begin{split}
&\mathcal{L}(L^{t+1}, S^{t+1}, Y^{t}, \mu^t)  \\
&\leq \mathcal{L}\left(L^1, S^{1},Y^{0}, \mu^{0}\right)+\sum_{i=1}^t\frac{\mu^i+\!\mu^{i-1}}{2(\mu^{i-1})^2}\|Y^i-Y^{i-1}\|_F^2.
\end{split}
\end{equation}
Since $\|Y^i-Y^{i-1}\|_F^2$ is bounded,  
all terms on the right-hand side of the above inequality are bounded, thus $\mathcal{L}\left(L^{t+1}, S^{t+1},Y^t, \mu^{t}\right)$ is upper bounded.

Again,
\begin{equation}
\begin{split}
\label{boundseq}
& \mathcal{L}\left(L^{t+1}, S^{t+1},Y^t, \mu^{t}\right)+\frac{1}{2\mu^t}\|Y^t\|_F^2\\
&=F(L^{t+1})+\!\lambda \|S^{t+1}\|_l+\!\frac{\mu^t}{2}\|L^{t+1}-\!X+\!S^{t+1}+\!\frac{Y^t}{\mu^t}\|_F^2.
\end{split}
\end{equation}
Since each term on the right-hand side is bounded, $S^{t+1}$ is bounded. By the last term on the right-hand of (\ref{boundseq}), $L^{t+1}$ is bounded. Therefore, $\{L^t\}$ and $\{S^t\}$ are both bounded.
\end{proof}
\begin{theorem}
Let $\{L^t, S^t, Y^t\}$ be the sequence generated in Algorithm 1 and $\{ L^*, S^*, Y^* \}$ be an accumulation point. Then $\{ L^*, S^* \}$ is a stationary point of optimization problem (\ref{originalproblem}) if $\sum\limits_{t=1}^{\infty} \frac{\mu^t+\mu^{t+1}}{(\mu)^2}<\infty$ and $\lim\limits_{t\to  \infty}\mu^t (S^{t+1} - S^{t}) \to 0$. 
\end{theorem}
\begin{proof}
The sequence $\{L^t, S^t, Y^t\}$ generated in Algorithm 1 is bounded as shown in Lemmas
\ref{lemma11} and \ref{lemma12}. By Bolzano-Weierstrass theorem, the sequence must have at least one accumulation point, e.g., $\{L^*, S^*, Y^*\}$. Without loss of generality, we assume that 
$\{  L^t, S^t, Y^t\}$ itself converges to $\{L^*, S^*,Y^*\}$.

Since $L^t+S^t-X=\frac{Y^t-Y^{t-1}}{\mu^{t-1}}$, we have $\lim\limits_{t\to  \infty}L^t+S^t-X \rightarrow 0$. Then $X=L^*+S^*$. Thus the primal feasibility condition is satisfied.

For $L^{t+1}$, it is true that
\begin{equation}
\begin{split}
\label{bound1}
&\partial_L \mathcal{L}\left(L, S^t, Y^t, \mu^{t}\right)|_{L^{t+1}}\\
=&\partial_L F\left(L\right)|_{L^{t+1}}+Y^t+\mu^{t}\left(L^{t+1}-X+S^t\right)\\
=&\partial_L F\left(L\right)|_{L^{t+1}}+Y^{t+1}-\mu^t(S^{t+1}-S^{t})\\
=&0.
\end{split}
\end{equation}
If the singular value decomposition of $L$ is $U diag(\sigma_i)V^T$, according to Theorem \ref{lewis} in Appendix B, 
\begin{equation}
\begin{split}
\partial_L F\left(L\right)|_{L^{t+1}}=Udiag(\theta)V^T,
\end{split}
\end{equation}
where $\theta_i=\frac{(1+\gamma)\gamma}{(\gamma+\sigma_i)^2}$ if $\sigma_i \neq 0$; otherwise, it is $(1+\gamma)/\gamma$.
Since $\theta_i\in (0, \frac{1+\gamma}{\gamma}]$ is finite, 
$\partial_L F\left(L\right)|_{L^{t+1}}$ is bounded. Since $\{Y^t\}$ is bounded, $\lim\limits_{t\to \infty} \mu^t ( S^{t+1} - S^{t})$ is bounded. 
Under the assumption that  $\lim\limits_{t\to \infty} \mu^t ( S^{t+1} - S^{t}) \to 0$ \cite{nie2014new}, 
\begin{equation*}
\partial_L F\left(L^*\right)+Y^*=0.
\end{equation*} 
Hence, $\{L^*, S^*, Y^*\}$ satisfies the KKT conditions of $\mathcal{L} (L, S, Y)$. Thus $\{L^*, S^*\}$ is a stationary point of (\ref{originalproblem}). 
\end{proof}

\section{Experiments}
In this section, we evaluate our algorithm by deploying it to three important practical applications: foreground-background separation on surveillance video, shadow removal from face images, and anomaly detection. All applications involve the recovery of intrinsically low-dimensional data from gross corruption. We compare our algorithm with other state-of-the-art methods, including convex RPCA \cite{candes2011robust}, CNorm \cite{sun2013robust}, and AltProj \cite{netrapalli2014non}. All these three methods call PROPACK package \cite{ding2011bayesian}. In addition, for convex RPCA \cite{wright2009robust}, we use the state-of-the-art solver, viz., an inexact augmented Lagrange multiplier (IALM) method \cite{candes2011robust}, \cite{lin2010augmented}. For CNorm, we use the fast alternating algorithm and the convex relaxation solutions from NSA \cite{aybat2011fast} as its initial conditions. We perform all experiments with Matlab in Windows 7 based on Intel Xeon 2.33GHz CPU with 4G RAM\footnote{The code is available at: https://github.com/sckangz/noncvx-PRCA}. 

\subsection{Parameter setting}
There are three parameters in our model: $\lambda$, $\rho$, and $\mu$. If $\lambda$ is too large, the trivial solution of $S=0$ is obtained, which generates $L$ with high rank. On the other hand, small $\lambda$ leads to $L=0$. Similar to \cite{candes2011robust}, $\lambda$ can be selected from the neighborhood of $1/\sqrt{\max\{m, n\}}$. Experiments show that our results are insensitive to $\lambda$ in a pretty broad range, so we just set $\lambda=10^{-3}$ through all our experiments. As for $\rho$, a large value will lead to fast convergence, while a small value of $\rho$ will result in a more accurate solution. In the literature, a often used value is 1.1. As discussed in \cite{leow2013background}, $\mu^0$ can also affect $L$ in IALM. If $\mu^0$ is too large, $L$ will have a rank larger than the desired low rank. This provides a way to manipulate the desired rank. By this principle, we tune the value of $\mu^0$. Finally, we set $\mu^0$ to $10^{-4}$, $3\times 10^{-3}$, $0.5$ and 4 for the following four experiments, respectively. In practice, these parameters can be chosen by cross validation. For fair comparison, we follow the experimental setting in AltProj \cite{netrapalli2014non} and stop the program when a relative error $\|X-L-S\|_F/\|X\|_F$ of $10^{-3}$ is reached. For other methods, we follow the parameter settings used in the corresponding papers. 
\begin{figure*}[t]
\centering
\begin{tabular}{cccc}
\hspace*{-7pt}\includegraphics[width=0.22\textwidth]{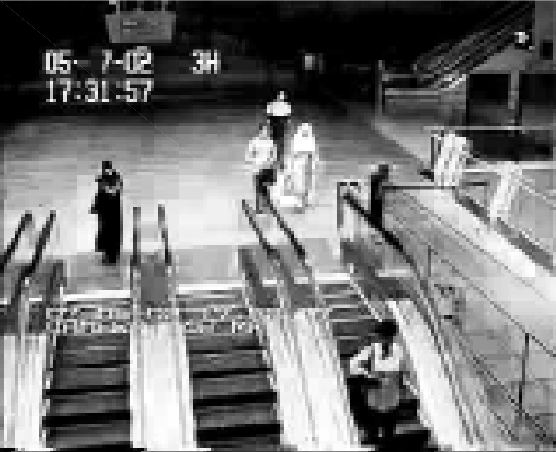}&
\hspace*{5pt}\includegraphics[width=0.22\textwidth]{elevator_X1}&
\hspace*{5pt}\includegraphics[width=0.22\textwidth]{elevator_X1}&
\hspace*{5pt}\includegraphics[width=0.22\textwidth]{elevator_X1}\\
\hspace*{-7pt}\includegraphics[width=0.22\textwidth]{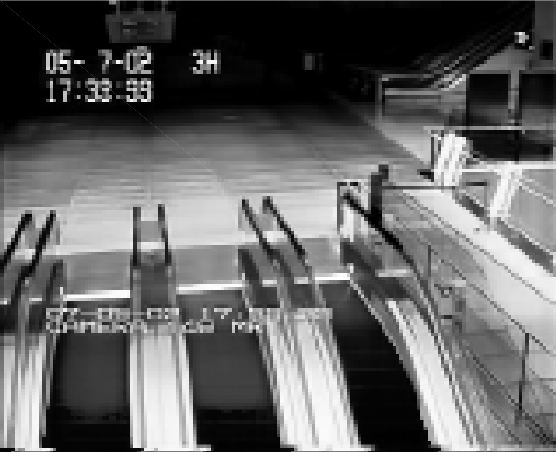}&
\hspace*{5pt}\includegraphics[width=0.22\textwidth]{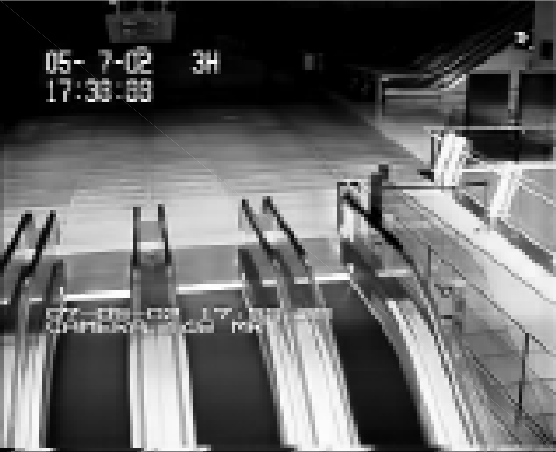}&
\hspace*{5pt}\includegraphics[width=0.22\textwidth]{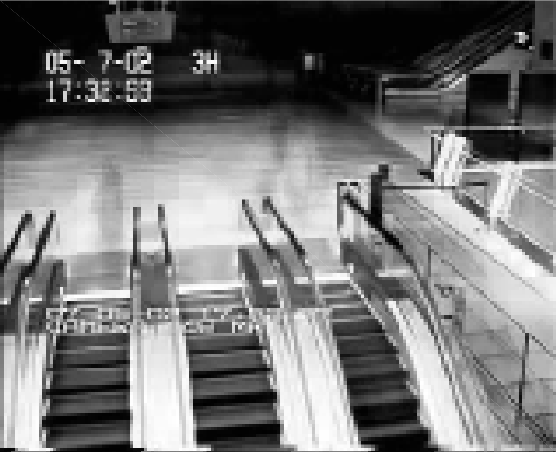}&
\hspace*{5pt}\includegraphics[width=0.22\textwidth]{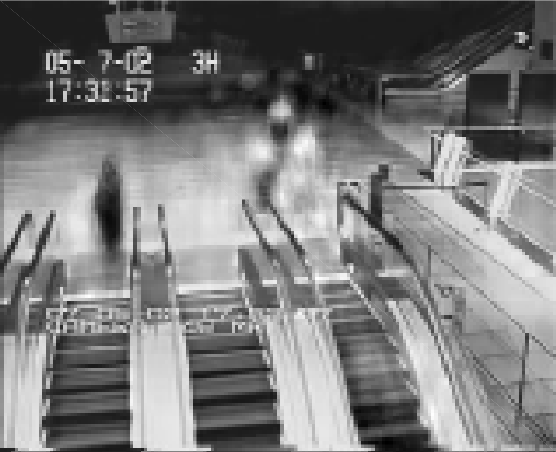}\\
\hspace*{-7pt}\includegraphics[width=0.22\textwidth]{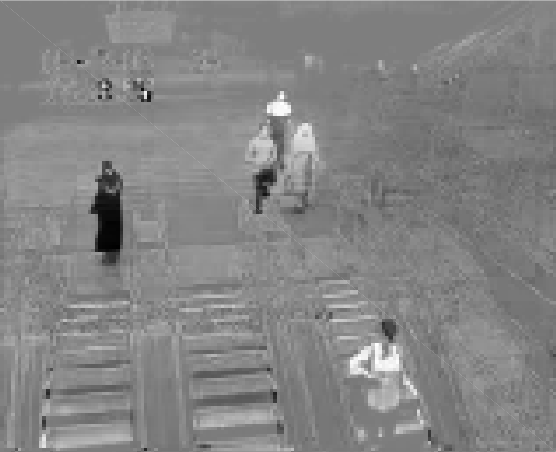}&
\hspace*{5pt}\includegraphics[width=0.22\textwidth]{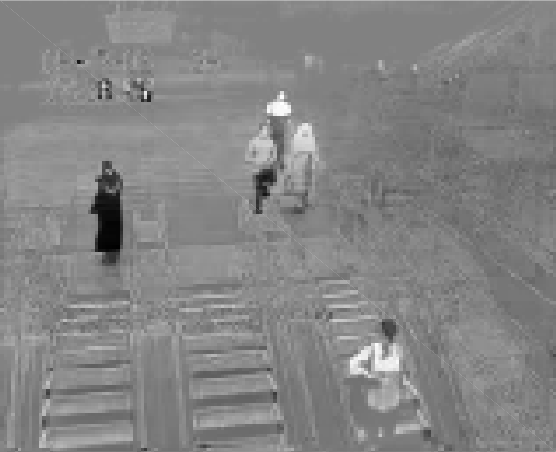}&
\hspace*{5pt}\includegraphics[width=0.22\textwidth]{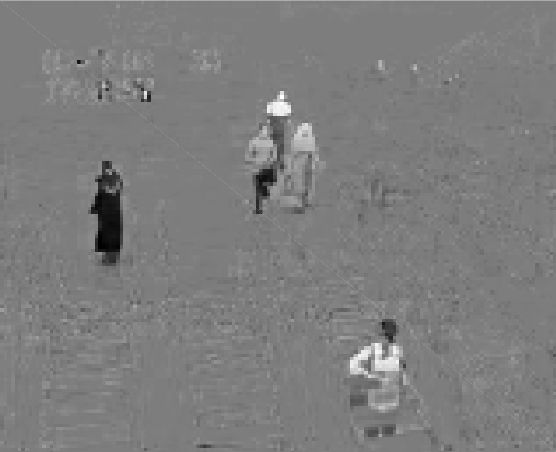}&
\hspace*{5pt}\includegraphics[width=0.22\textwidth]{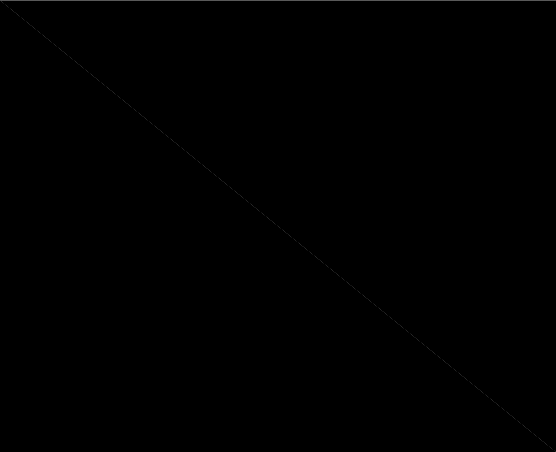}\\
(a) AltProj&(b) Our&(c) IALM&(d) CNorm
\end{tabular}\vspace*{-5pt}
\caption{Foreground-background separation in the escalator video. The three rows from top to bottom are original image frame, static background, and dynamic foreground, respectively.}
\label{elevator}
\end{figure*}
\subsection{Video background subtraction}
Background subtraction from video sequences is a popular approach to detecting interesting activities in the scene. Surveillance videos from a fixed camera can be naturally modeled by our model due to their relatively static background and sparse foreground. 
\subsubsection{First experiment scenario}
In this experiment, we use a benchmark data set escalator \cite{li2004statistical}, which contains 3,417 frames of size 160 $\times$ 130. The data matrix $X$ is formed by vectorizing each frame and concatenating the vectors column-wisely. As a result, the size of $X$ is $20,800\times 3,417$. For this data set, the background appears to be completely static, thus the ideal rank of the background matrix is one.

\begin{table}[htbp]
\renewcommand{\arraystretch}{1.3}

\caption{Recovery results of escalator video}
\centering
\begin{tabular}{|c||c|c|c|}
\hline
Algorithm & Rank($L$) &$\frac{\|X-L-S\|_F}{\|X\|_F}$&Time (s)\\
\hline
AltProj & 1& 6.35e-4&537\\
\hline
CNorm & 131& 1.00e-1&1015\\
\hline
IALM& 2011&9.50e-8&11315\\
\hline
Our & 1& 5.45e-4&208\\
\hline
\end{tabular}
\label{videoresult}
\end{table}

For escalator, unfortunately, CNorm cannot successfully separate the foreground from background with our current stopping criterion. Thus we relax its terminating relative error to be $0.1$. For AltProj, we set the desired rank of $L^*$ to be one. The low rank ground truth is not available for
these videos so we present a visual comparison of extraction results using different algorithms in Figure \ref{elevator}. As we can see, CNorm suffers from noticeable artifacts (shadows of people), which is due to overfitting in the low rank component. $S$ is missing since the sparse component is absorbed by the big relative error.  Blur exists in IALM recovery image, which is also observed in many other work \cite{netrapalli2014non}, \cite{babacan2012sparse}. In contrast, both AltProj and our method obtain a clean background. Moreover, the steps of the escalator are also removed by these two methods, since they are moving and are supposed to be part of the dynamic foreground. 

Table \ref{videoresult} gives the quantitative comparison results. In terms of computing time, our method is more than twice faster than AltProj, and 54 times faster than IALM\footnote{The experiments in \cite{netrapalli2014non} are conducted on a machine with Dual 8-core Xeon (E5-2650) 2.0GHz CPU with 192GB RAM \cite{netrapalli2014nonarxiv}. }. One intuitive interpretation is that we observe that fewer iterations are required for our algorithm to converge. Therefore our method is efficient even when the matrix size is large. Furthermore, both AltProj and our algorithm can obtain the desired rank-one matrix $L$. The nuclear norm based IALM results in $L$ with rank 2011, which implies that $L$ contains many blurred images. If we increase the size of data matrix $X$, IALM performs even worse since it may incur errors from rank approximation. These results emphasize the significance of good rank approximation.  
\begin{figure*}[ht]
\centering
\subfigure[ AltProj]
{
\begin{minipage}{.32\textwidth}
\hspace*{-5pt}\includegraphics[width=0.45\textwidth]{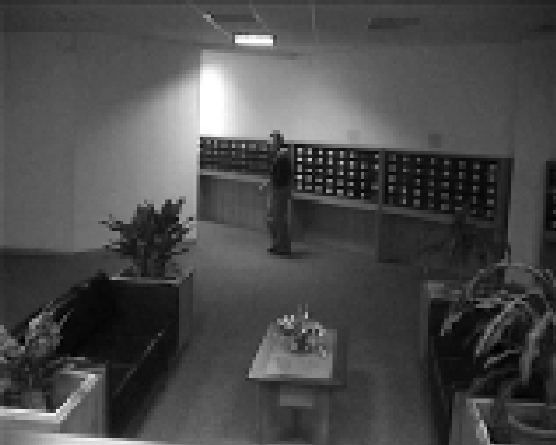}
\vspace{1mm}
\hspace*{1pt}\includegraphics[width=0.45\textwidth]{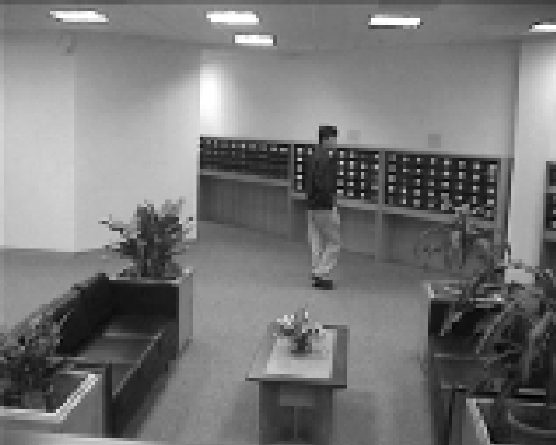}\\
\hspace*{-5pt}\includegraphics[width=0.45\textwidth]{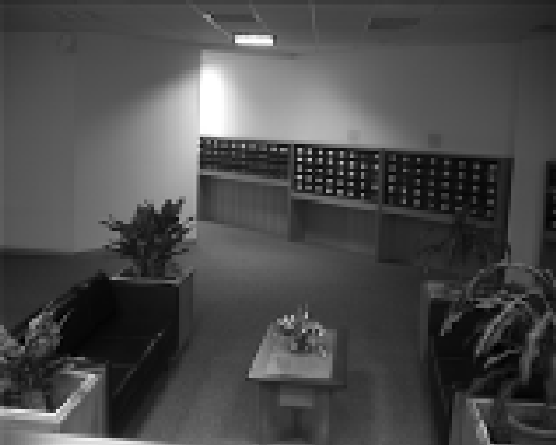}
\vspace{1mm}
\hspace*{1pt}\includegraphics[width=0.45\textwidth]{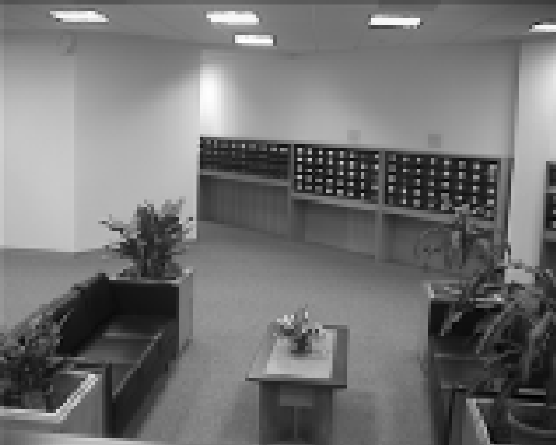}\\
\hspace*{-5pt}\includegraphics[width=0.45\textwidth]{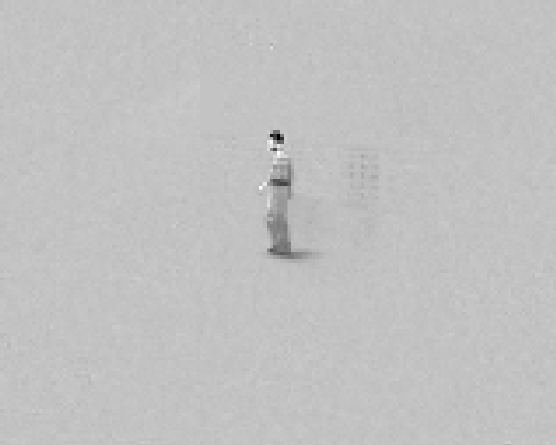}
\vspace{1mm}
\hspace*{1pt}\includegraphics[width=0.45\textwidth]{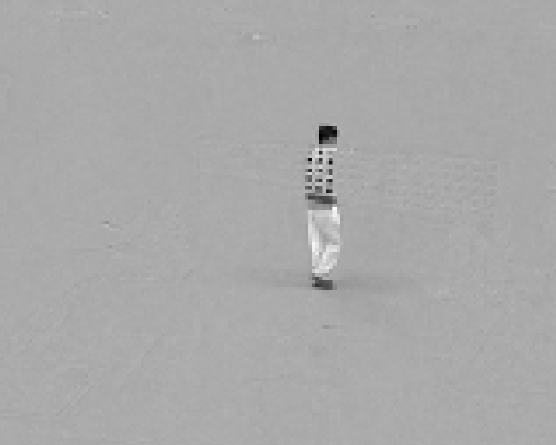}
\end{minipage}
}
\subfigure[ IALM]
{\begin{minipage}{.32\textwidth}
\hspace*{-5pt}\includegraphics[width=0.45\textwidth]{lobby_1010}
\vspace{1mm}
\hspace*{1pt}\includegraphics[width=0.45\textwidth]{lobby_1371}\\
\hspace*{-5pt}\includegraphics[width=0.45\textwidth]{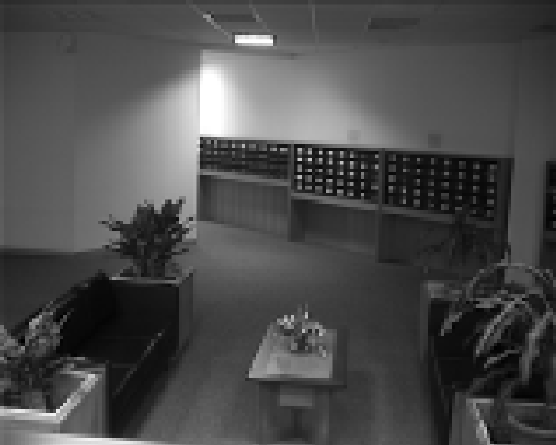}
\vspace{1mm}
\hspace*{1pt}\includegraphics[width=0.45\textwidth]{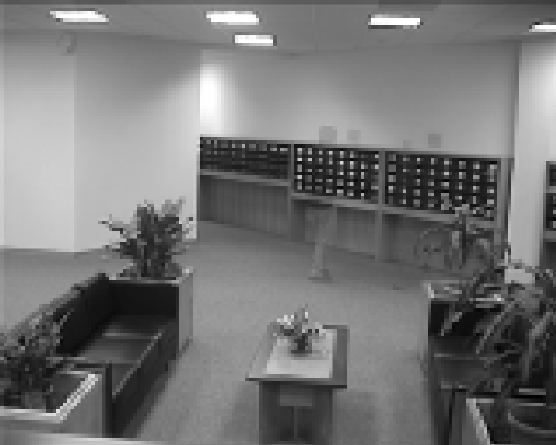}\\
\hspace*{-5pt}\includegraphics[width=0.45\textwidth]{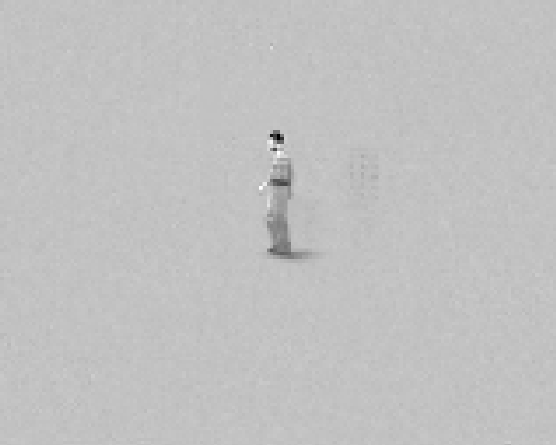}
\vspace{1mm}
\hspace*{1pt}\includegraphics[width=0.45\textwidth]{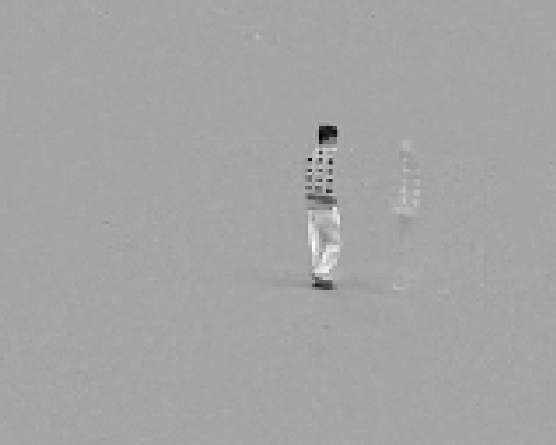}
\end{minipage}
}
\subfigure[ Our]
{\begin{minipage}[]{.32\textwidth}
\hspace*{-5pt}\includegraphics[width=0.45\textwidth]{lobby_1010}
\vspace{1mm}
\hspace*{1pt}\includegraphics[width=0.45\textwidth]{lobby_1371}\\
\hspace*{-5pt}\includegraphics[width=0.45\textwidth]{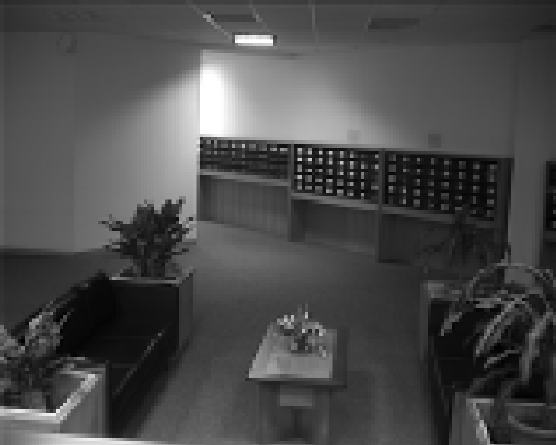}
\vspace{1mm}
\hspace*{1pt}\includegraphics[width=0.45\textwidth]{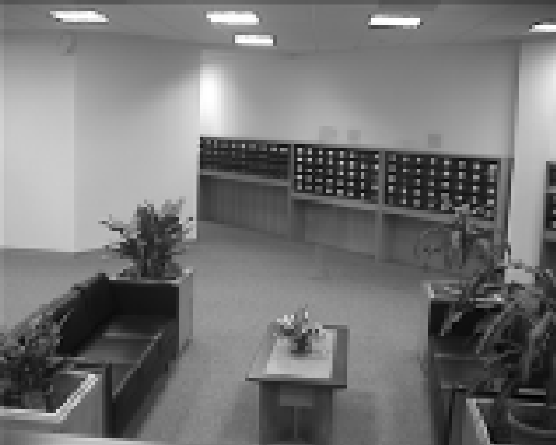}\\
\hspace*{-5pt}\includegraphics[width=0.45\textwidth]{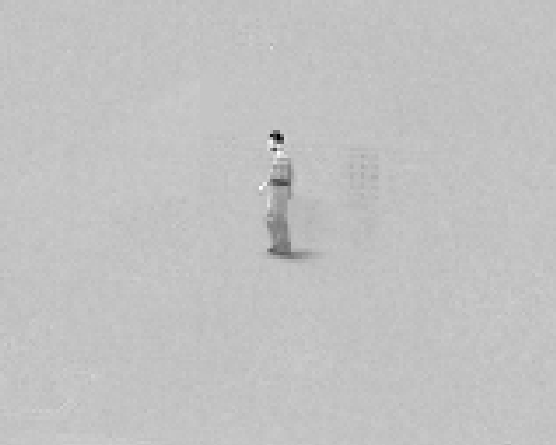}
\vspace{1mm}
\hspace*{1pt}\includegraphics[width=0.45\textwidth]{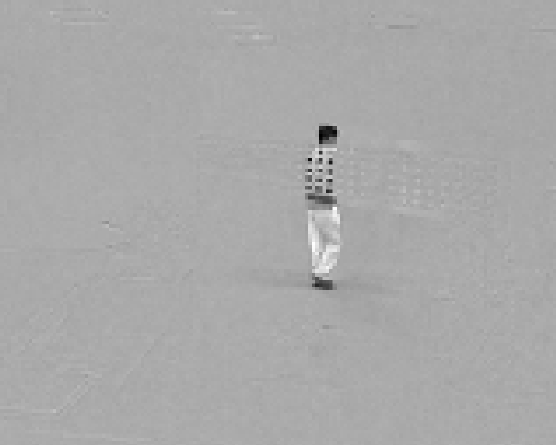}
\end{minipage}
}
\caption{Foreground-background separation in the lobby video. The three rows from top to bottom are original image frame,  background, and dynamic foreground, respectively.}
\label{lobby}
\end{figure*}
\subsubsection{Second experiment scenario}

\begin{table}[htbp]
\renewcommand{\arraystretch}{1.3}
\caption{Recovery results of lobby video}
\label{lobbyresult}
\centering
\begin{tabular}{|c||c|c|c|}
\hline
Algorithm & Rank($L$) &$\frac{\|X-L-S\|_F}{\|X\|_F}$&Time (s)\\
\hline
AltProj & 2& 1.88e-4&203\\
\hline
IALM& 259&9.59e-4&1187\\
\hline
Our & 2& 1.95e-5&46\\
\hline
\end{tabular}
\end{table}
The purpose of this experiment is to demonstrate the effectiveness of our algorithm on dynamic backgrounds. In some cases, the background changes over time due to, e.g., illumination variation or weather change. Then the background can have a higher rank. Here we use sequences captured from a lobby. The size of $X$ is $20,480 \times 1,546$. On this occasion, background changes are mainly caused by switching on/off lights. Therefore, we expect the rank of $L$ to be 2. Again, we set the desired rank in Altproj to be 2. Two examples are shown for each method in Figure \ref{lobby}. The first example denotes the scene before the other three lights are turned on. In this example, except for some shadows of pants in one image recovered by IALM, all the recovered background images appear satisfactory.

We list the numerical results in Table \ref{lobbyresult}. For this experiment, our algorithm is almost five times faster than AltProj, 26 times faster than IALM. It is also noted that, the rank obtained from IALM is still high.   
\begin{figure*}[htbp]
\centering
\begin{tabular}{cccccccc}
\hspace*{-7pt}\includegraphics[width=0.18\textwidth]{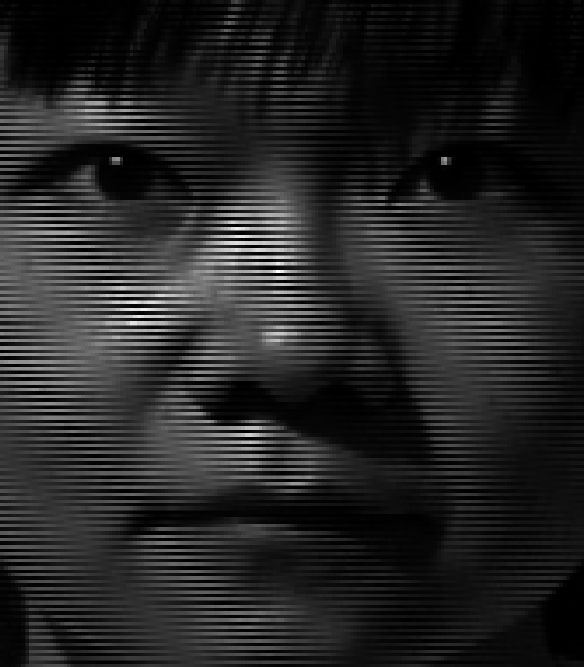}&
\hspace*{5pt}\includegraphics[width=0.18\textwidth]{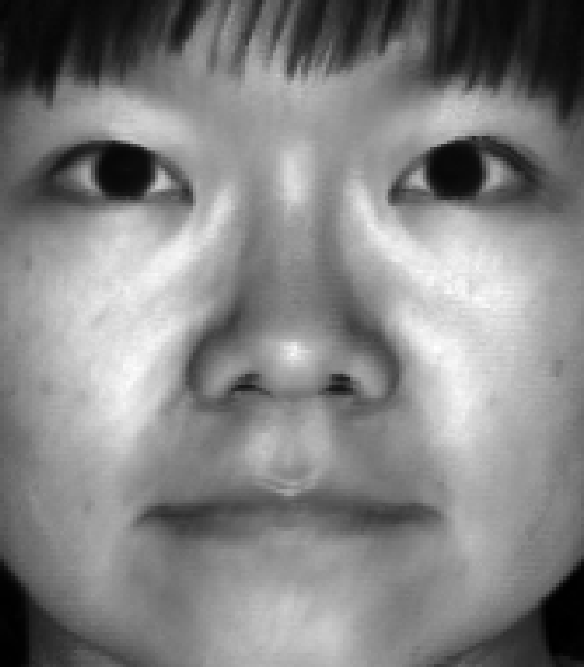}&
&
\hspace*{-5pt}\includegraphics[width=0.18\textwidth]{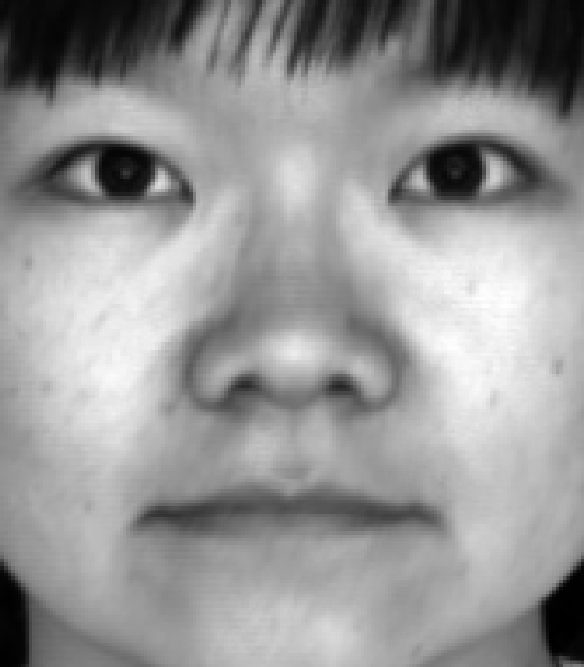}&
&
\hspace*{-5pt}\includegraphics[width=0.18\textwidth]{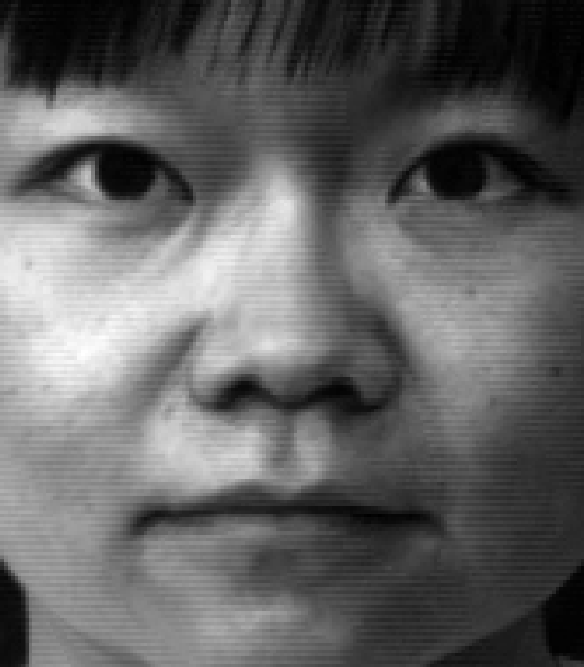}&
&
\hspace*{-5pt}\includegraphics[width=0.18\textwidth]{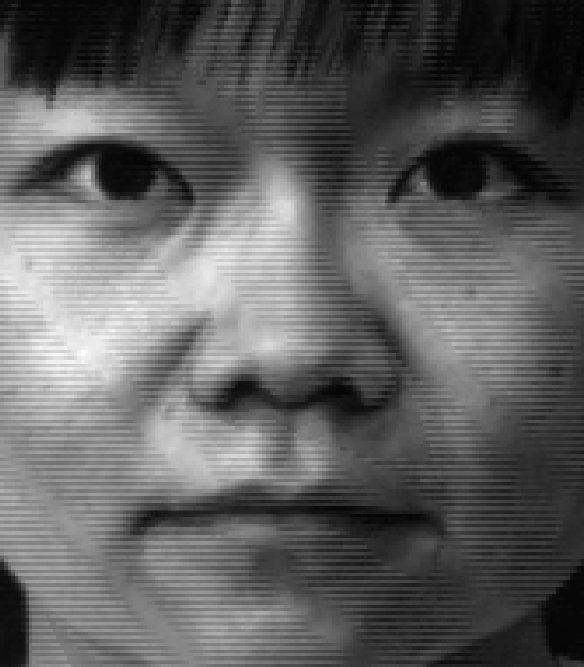}\\
&
\hspace*{5pt}\includegraphics[width=0.18\textwidth]{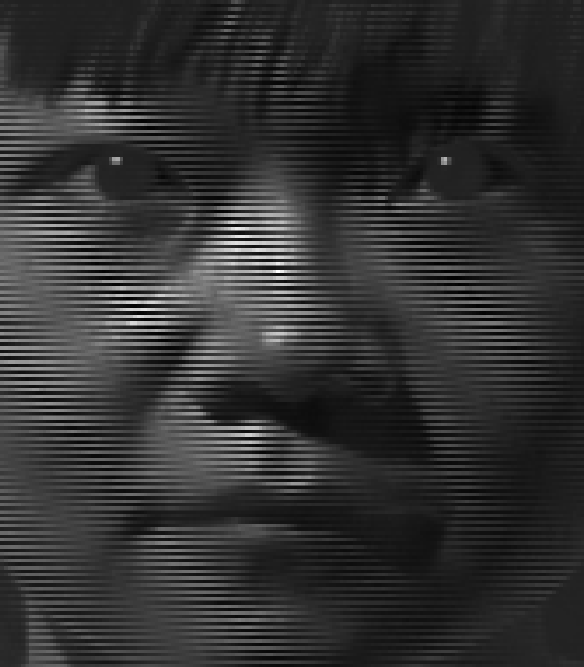}&
&
\hspace*{-5pt}\includegraphics[width=0.18\textwidth]{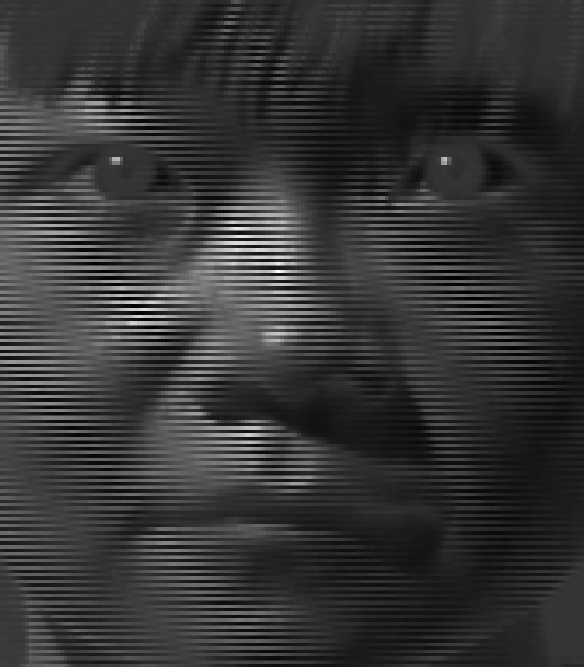}&
&
\hspace*{-5pt}\includegraphics[width=0.18\textwidth]{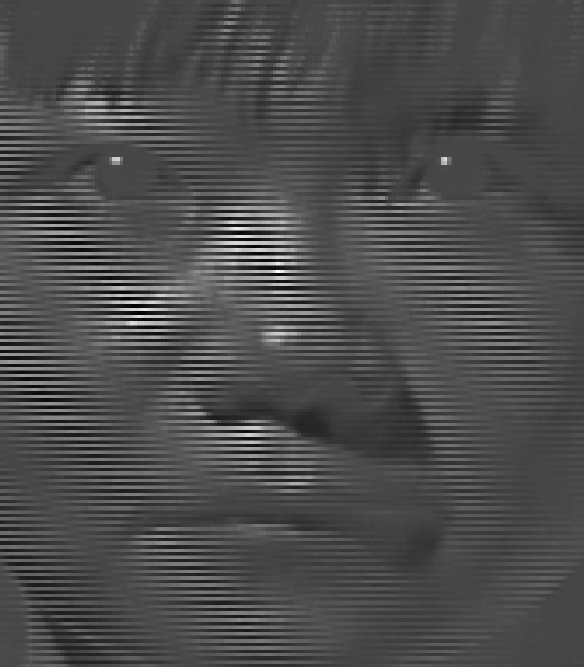}&
&
\hspace*{-5pt}\includegraphics[width=0.18\textwidth]{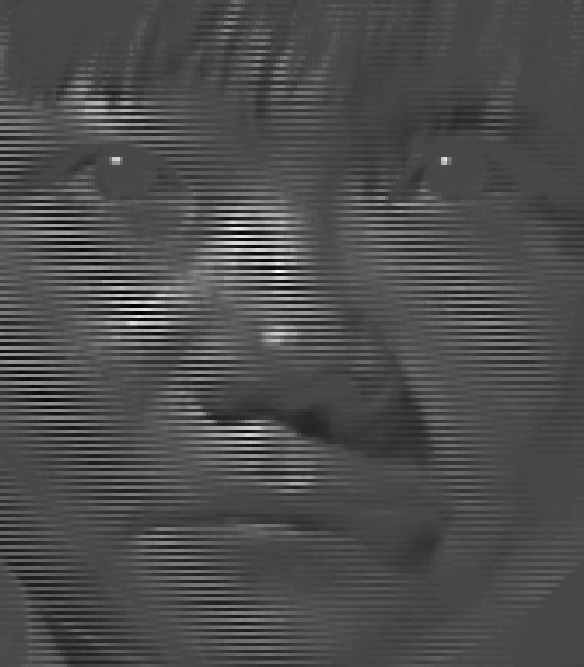}\\
\vspace*{5pt}\\
\hspace*{-7pt}\includegraphics[width=0.18\textwidth]{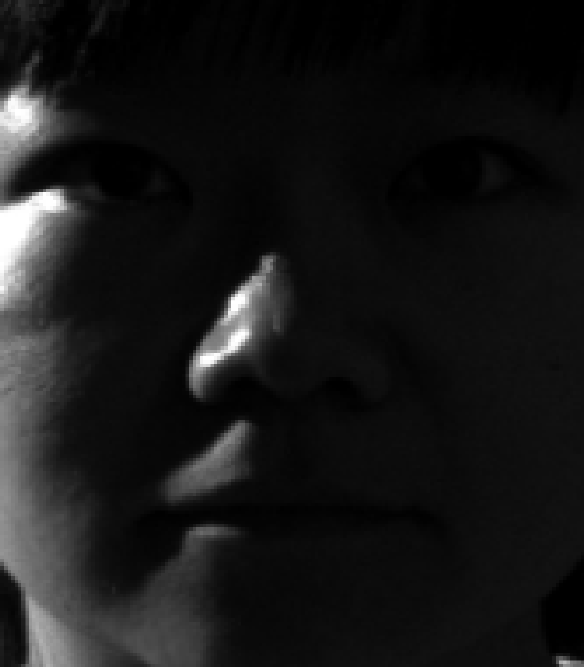}&
\hspace*{5pt}\includegraphics[width=0.18\textwidth]{subject5_ncx30}&
&
\hspace*{-5pt}\includegraphics[width=0.18\textwidth]{subject5_our}&
&
\hspace*{-5pt}\includegraphics[width=0.18\textwidth]{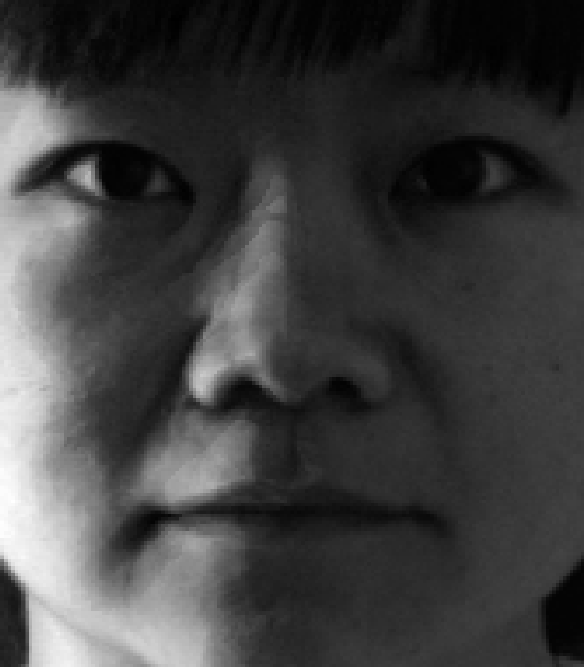}&
&
\hspace*{-5pt}\includegraphics[width=0.18\textwidth]{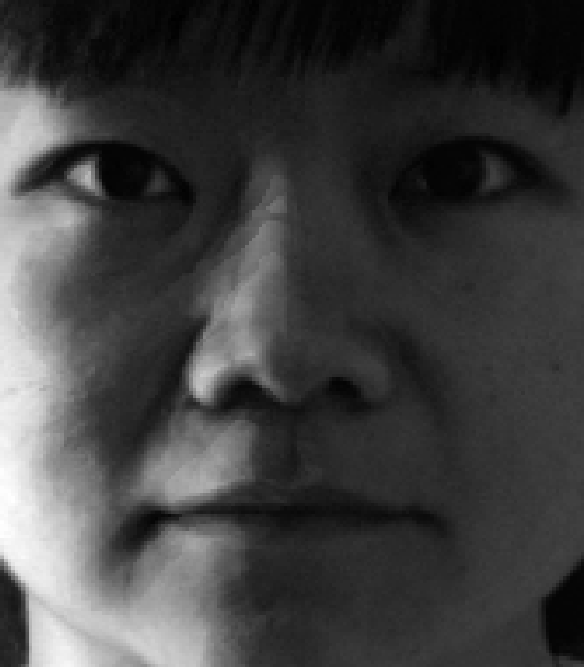}\\
&
\hspace*{5pt}\includegraphics[width=0.18\textwidth]{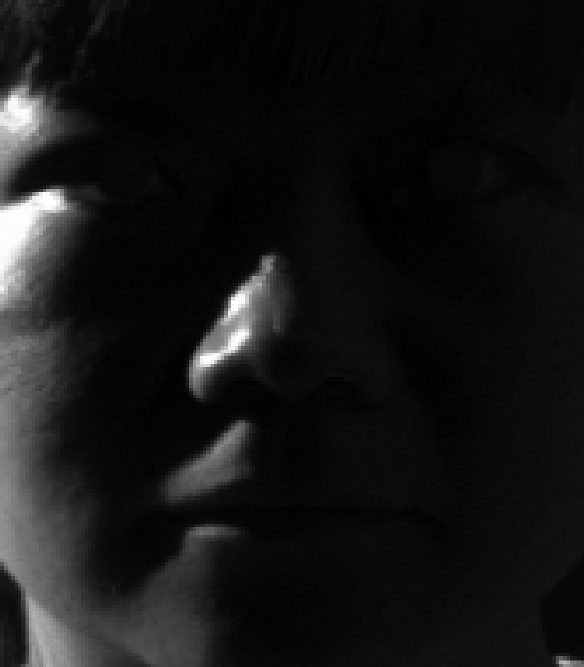}&
&
\hspace*{-5pt}\includegraphics[width=0.18\textwidth]{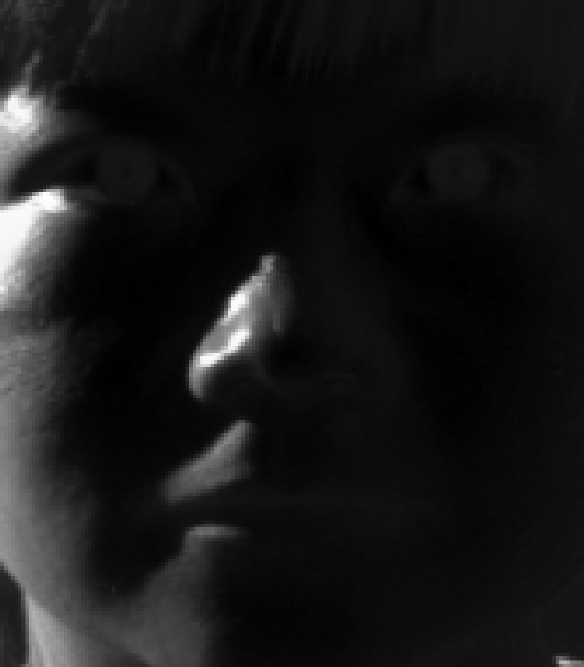}&
&
\hspace*{-5pt}\includegraphics[width=0.18\textwidth]{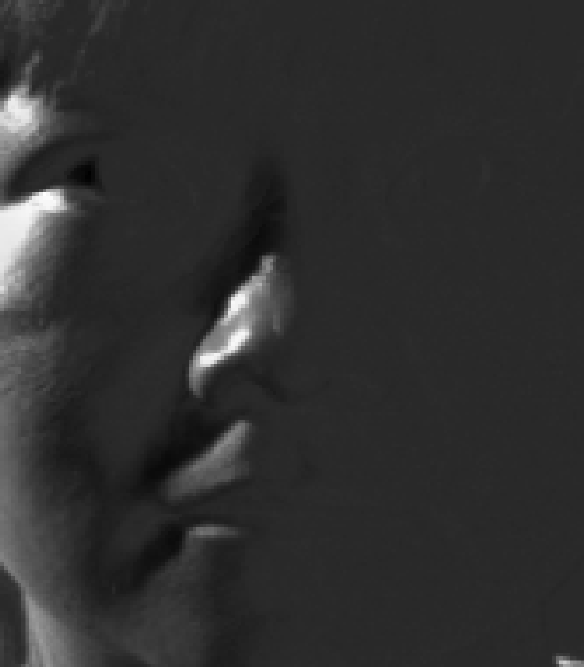}&
&
\hspace*{-5pt}\includegraphics[width=0.18\textwidth]{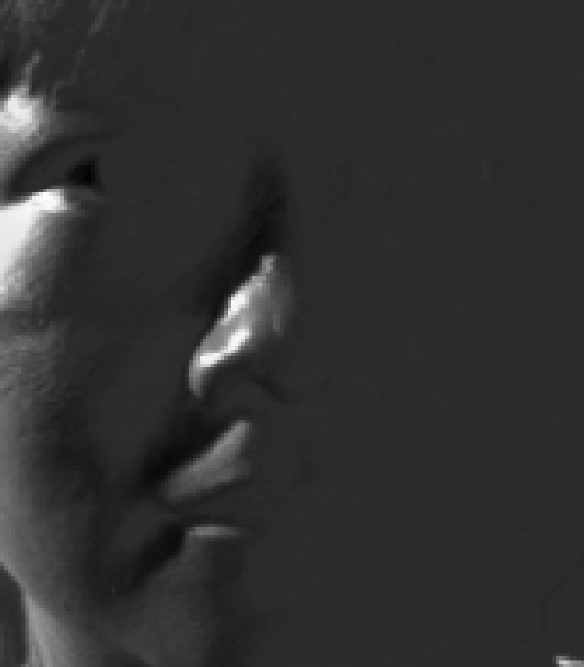}\\
 (a) Original images& (b) AltProj&&(c) Our&&(d) IALM&&(e) CNorm
\end{tabular}\vspace*{-5pt}
\caption{Shadow removal from face images. Column 1 displays sample images 17 and 30 from Subject05 of the Yale B database. Columns 2 to 5 show their low rank approximation obtained by different algorithms. Rows 2 and 4 are corresponding sparse components.}
\label{face}
\end{figure*}

\subsection{Face image shadow removal}
Removing shadows, specularities, and saturations from face images is another important application of RPCA \cite{candes2011robust}. Face images taken under different lighting conditions often introduce errors to face recognition \cite{basri2003lambertian}. 
These errors might be large in magnitude, but are supposed to be sparse in the spatial domain. Given enough face images of the same person, it is possible to reconstruct the true face image. 

We use images from the Extended Yale B database \cite{lee2005acquiring}. There are 38 subjects and each subject has 64 images of size $192\times168$ taken under varying different illuminations. Images of each subject are heavily corrupted due to different illumination conditions. All images are converted to 32,256-dimensional column vectors, hence $X\in\mathcal{R}^{32256\times64}$ for each subject. Since the images are well aligned, $L$ should have a rank of 1.

%
%

\begin{table}[htbp]
\renewcommand{\arraystretch}{1.3}
\caption{Extended Yale B Face images recovery results}
\centering
\begin{tabular}{|c||c|c|c|}

\hline
Algorithm & Rank($L$) &$\frac{\|X-L-S\|_F}{\|X\|_F}$&Time (s)\\
\hline
AltProj & 1& 4.88e-4&22\\
\hline
CNorm & 26& 1.00e-3&138\\
\hline
IALM& 32&6.40e-4& 9\\
\hline
Our & 1& 3.07e-5&0.5\\
\hline
\end{tabular}
\label{faceresult}
\end{table}
%
\begin{figure}[htbp]
\includegraphics[width=.45\textwidth]{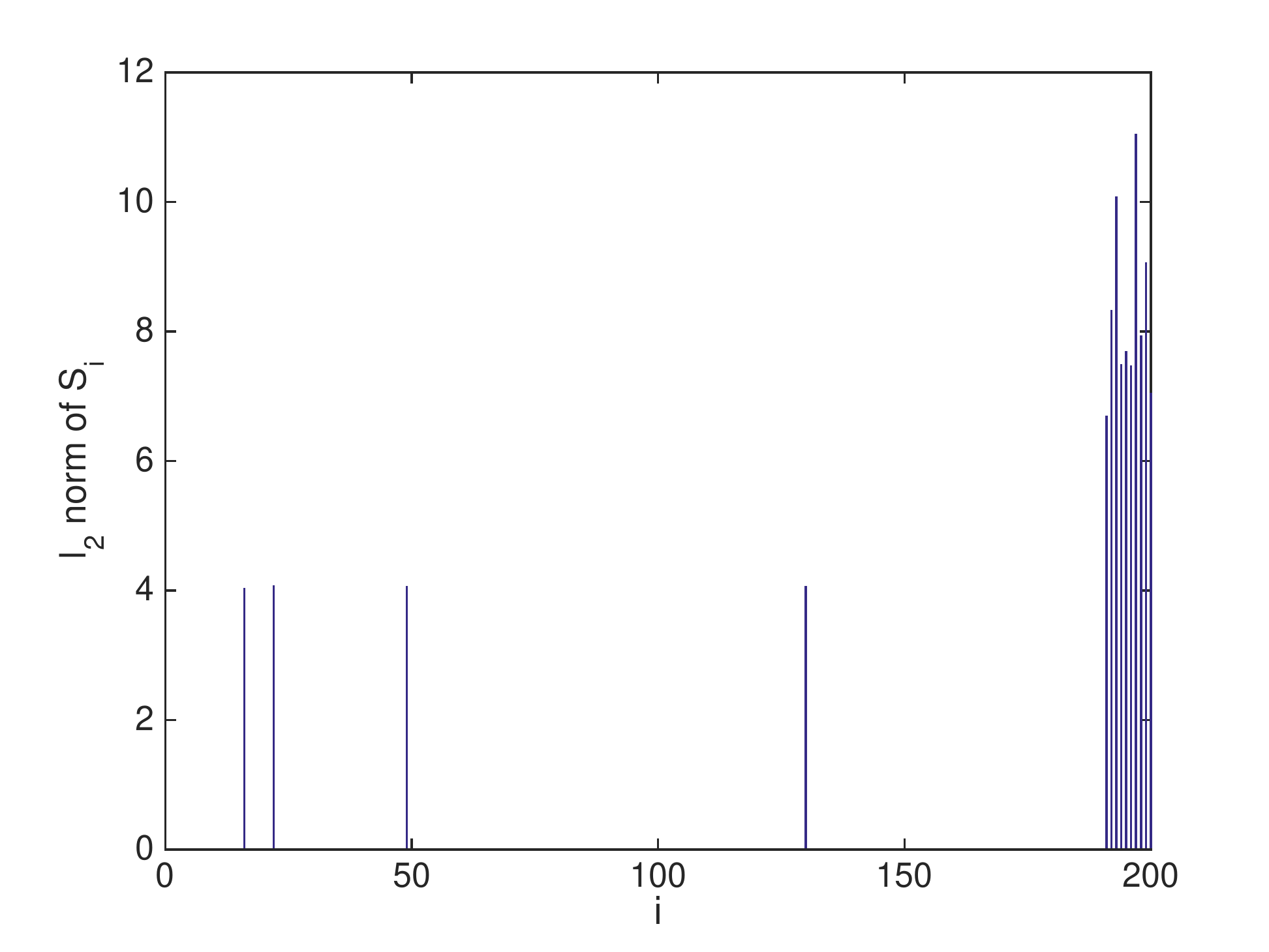}
\caption{$\ell_2$ norm of each of the 200 columns of $S$.}
\label{bar}
\end{figure}

Figure \ref{face} illustrates the results of different algorithms on two images. The proposed algorithm removes the specularities and shadows well, while there are some artifacts left by using IALM and CNorm. Although the visual qualities are similar for AltProj and our method, the numerical measurements in Table \ref{faceresult} demonstrate that our algorithm is 22 times faster than AltProj. Similar to the results in \cite{sun2013robust}, IALM and CNorm result in $L$ of high ranks.
\subsection{Anomaly Detection}
It is widely known that images from the same subject reside in a low-dimensional subspace. If we inject some images of  different subjects into a dominant number of images of the same subject, they will stand out as outliers or anomalies. To test this, we use images from USPS data set \cite{rasmussengaussian}. We choose digits '1' and '7', since they share some similarities. And we treat each 16$\times$16 image as a 256-dimensional column vector. Then the data matrix $X$ is constructed to contain the first 190 samples from digit '1' and the last 10 samples from '7'. Our goal is to identify all anomalies, including all '7's, in an unsupervised way. We apply our model to estimate $L$ and $S$, which are expected to capture '1's and '7's, respectively. $\ell_2$ norm of each column in $S$ is used to identify anomalies. Ideally, '7's should
give larger values than '1's.

\begin{figure}[htbp]
\includegraphics[width=.45\textwidth]{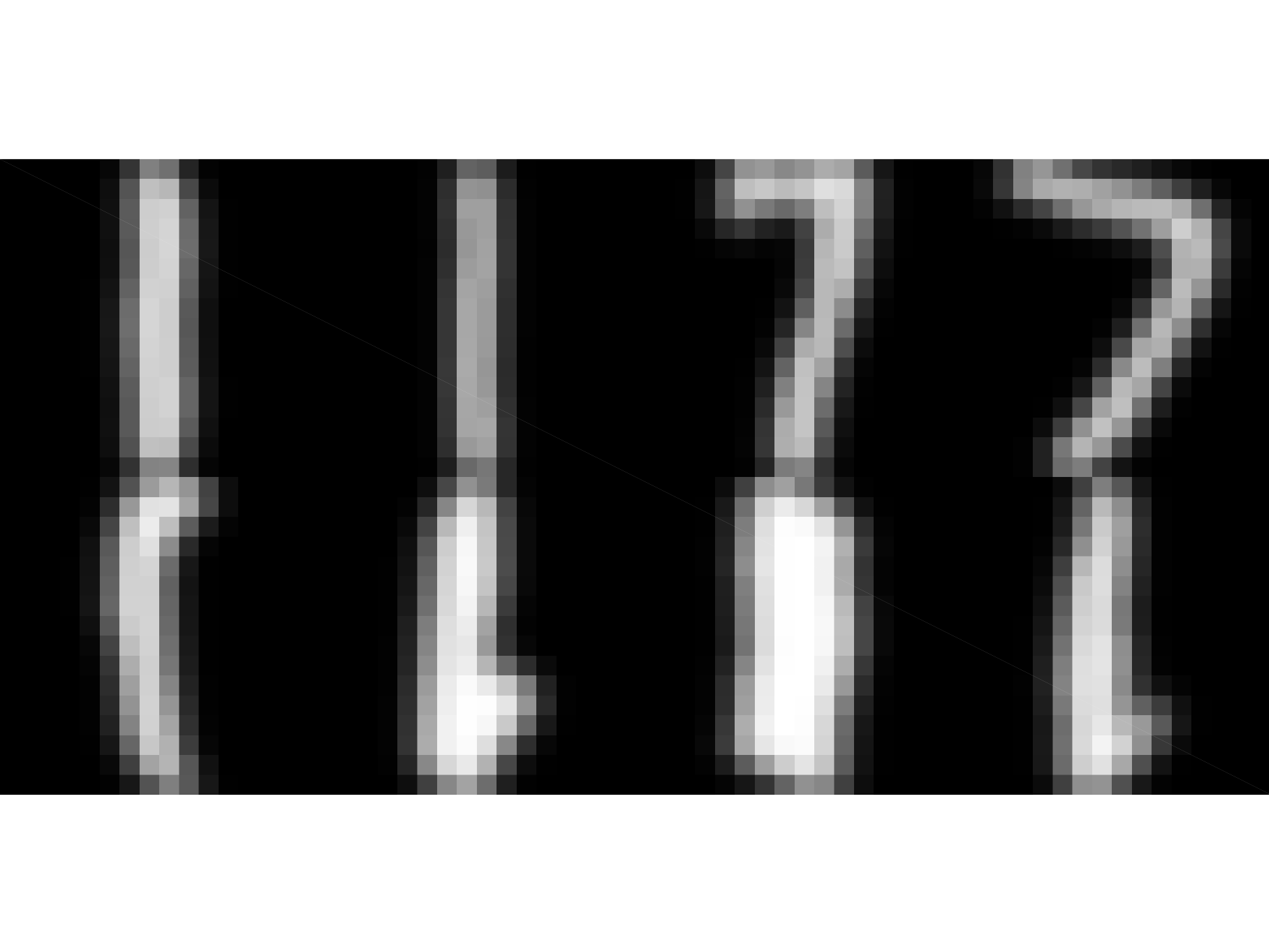}
\caption{USPS anomaly detection results. The first row gives some typical '1's and '7's. The second row plots the four abnormal '1's identified in Figure \ref{bar}.}
\label{anomalysample}
\end{figure}


Figure \ref{bar} shows the $l_2$ norm of columns in $S$. For visual quality, we apply thresholding with a threshold of 4 to get rid of small values. We can see that all '7's are found. Besides, four '1's at columns 16, 22, 49 and 130 also appear. As shown in Figure \ref{anomalysample}, these four '1's are written in a way different from the rest of '1's. 

\section{Conclusion}
This paper investigates a nonconvex approach to the robust principal component analysis (RPCA) problem. In particular, we provide a novel matrix rank approximation, which is more robust and less biased than the nuclear norm. We devise an augmented Lagrange multiplier framework to solve this nonconvex optimization problem. Extensive experimental results demonstrate that our proposed approach outperforms previous algorithms. Our algorithm can be used as a powerful tool to efficiently separate low-dimensional and sparse structure for high-dimensional data. 
It would be interesting to establish more theoretical properties to the proposed nonconvex approach, for example, the theoretical guarantees for the estimator to be consistent. 

\section{Acknowledgments}
This work is supported by US National Science
Foundation Grants IIS 1218712. The corresponding author is Qiang
Cheng.

\appendices
\section{Proof of Theorem 1}
\setcounter{theorem}{0}
\begin{theorem}
\cite{kang2015robust}, \cite{peng2015subspace} Let $A=U\Sigma_A V^T$ be the SVD of $A\in \mathbf{\mathcal{R}}^{m\times n}$ and $\Sigma_A=diag(\sigma_A)$. Let $F(Z)=f\circ\sigma(Z)$ be a unitarily invariant function 
and $\mu>0$ 
. Then an optimal solution to the following problem

\begin{equation}
\min_Z F(Z)+\frac{\mu}{2}\left\|Z-A\right\|_F^2,
\label{theoremprob}
\end{equation}
is $Z^*= U\Sigma_Z^*V^T$, where $\Sigma_Z^*=diag(\sigma^*)$ and  $\sigma^* = \mathrm{prox}_{f, \mu} (\sigma_{A})$. Here $\mathrm{prox}_{f, \mu} (\sigma_{A})$ is the Moreau-Yosida operator, defined as 
\begin{equation}
\label{ascalar}
\mathrm{prox}_{f, \mu} (\sigma_A) := \argmin_{\sigma\geq0} f(\sigma) + \frac{\mu}{2}\|\sigma - \sigma_A\|_2^2.
\end{equation}
\end{theorem}
\begin{proof}
Since $A=U\Sigma_{A}V^T$, then $\Sigma_{A}=U^TAV$. Denoting $X=U^TZV$ which has exactly the same singular values as $Z$, we have
\begin{flalign}
&F(Z)+\frac{\mu}{2}\|Z-A\|_F^{2}&\label{lower}\\
&= F(X)+\frac{\mu}{2}\|X-\Sigma_A\|_{F}^{2},&\label{unitary}
\end{flalign}
\begin{flalign}
&\geq F(\Sigma_X)+\frac{\mu}{2}\|\Sigma_X-\Sigma_A\|_F^{2},&\label{hof}\\
&= F(\Sigma_Z)+\frac{\mu}{2}\|\Sigma_Z-\Sigma_A\|_F^{2},&\label{ten}\\
&= f(\sigma)+\frac{\mu}{2}\|\sigma-\sigma_{A}\|_2^2,&\label{eleven}\\
&\ge f(\sigma^*) + \frac{\mu}{2} \|\sigma^* - \sigma_{ A}\|_2^2. 
\end{flalign}
Note that (\ref{unitary}) hold since the Frobenius norm is unitarily invariant; (\ref{hof}) is due to the Hoffman-Wielandt inequality; and (\ref{ten}) holds as $\Sigma_X = \Sigma_Z$. Thus, (\ref{ten}) is a lower bound of (\ref{lower}). Because $\Sigma_Z = \Sigma_X = X = U^T Z V$, the SVD of $Z$ is $Z = U \Sigma_Z V^T$. By minimizing (\ref{eleven}), we get $\sigma^*$. Hence 
$Z^* = U diag(\sigma^*) V^T$, which  is the optimal solution of problem (\ref{theoremprob}).
\end{proof}

\section{}

\begin{lemma}
\label{21norm}
\cite{yuan2006model}\label{lemma:solve_l2l1} Let $H$ be a given matrix. If the optimal solution to
\begin{eqnarray*}
\min_{W}\vspace{.2cm} \alpha\left\|W\right\|_{2,1} + \frac{1}{2}\left\|W-H\right\|_F^2
\end{eqnarray*}
is $W^*$, then the $i$-th column of $W^*$ is
\begin{eqnarray*}
[W^*]_{:,i}=\left\{
\begin{array}{ll} \frac{\left\|H_{:,i}\right\|_2-\alpha}{\left\|H_{:,i}\right\|_2}H_{:,i}, & \mbox{if $\left\|H_{:,i}\right\|_2>\alpha$};\\
0, & \mbox{otherwise.}
\end{array}\right.
\end{eqnarray*}
\end{lemma}

\begin{lemma}
\label{1norm}
\cite{beck2009fast} The shrinkage-thresholding operator is defined as
\begin{equation*}
\begin{split}
\mathbf{\mathcal{T}}_{\lambda,h}(g)&:=\arg\min_{x}\frac{1}{2}hx^{2}+gx+\lambda|x|	\\
& = -\frac{1}{h}(|g|-\lambda)_{+}sign(g)\\
& =
\begin{cases}
-\frac{|g|-\lambda}{h}sign(g), &\mbox{if $|g|>\lambda$};\\
0, &\mbox{otherwise},
\end{cases}\\
\end{split}
\end{equation*}
where $h, g$ and $\lambda$ are scalars.  
\end{lemma}
\setcounter{theorem}{2}
\begin{theorem}
\label{lewis}
\cite{lewis2005nonsmooth} Suppose $F: \mathbf{\mathcal{R}}^{n_1\times n_2}\rightarrow \mathbf{\mathcal{R}}$ is represented as $F(X)=f \circ \sigma(X)$, where $X\in\mathbf{\mathcal{R}}^{n_1\times n_2} $ with SVD 
 $X=U diag(\sigma_1, \cdots, \sigma_n) V^T$, $n=\min(n_1, n_2)$, and $f$ is differentiable. The gradient of $F(X)$ at $X$ is
\begin{equation}
\label{deritheorem}
\frac{\partial F(X)}{\partial X}=U diag(\theta) V^T,
\end{equation}
where $\theta=\frac{\partial f(y)}{\partial y}|_{y=\sigma (X)}$.
\end{theorem}
%
\bibliographystyle{./IEEEtran}
\bibliography{./IEEEabrv,./pca}

\end{document}